%% file: example_paper.tex
\documentclass{article}

\usepackage{microtype}
\usepackage{graphicx}
\usepackage{subfigure}
\usepackage{booktabs} % for professional tables
\usepackage{longtable}
\usepackage{float}

\usepackage{hyperref}

\usepackage[accepted]{icml2025}

\usepackage{amsmath}
\usepackage{amssymb}
\usepackage{mathtools}
\usepackage{amsthm}

\usepackage[capitalize,noabbrev]{cleveref}

\theoremstyle{plain}
\newtheorem{theorem}{Theorem}[section]
\newtheorem{proposition}[theorem]{Proposition}

\theoremstyle{definition}

\theoremstyle{remark}

\usepackage[textsize=tiny]{todonotes}

\icmltitlerunning{Auditing Forgetting in Limited Memory Language Models}

\begin{document}

\twocolumn[
\icmltitle{Auditing Forgetting in Limited Memory Language Models}



\icmlsetsymbol{equal}{*}

\begin{icmlauthorlist}
\icmlauthor{Arya Raeesi}{yyy}
\icmlauthor{Hanna Roed}{yyy}
\end{icmlauthorlist}

\icmlaffiliation{yyy}{University of California, Berkeley, Berkeley, California, United States of America}

\icmlcorrespondingauthor{Arya Raeesi}{aryaraeesi@berkeley.edu}
\icmlcorrespondingauthor{Hanna Roed}{hanna.roed@berkeley.edu}

\renewcommand{\topfraction}{0.9}
\renewcommand{\bottomfraction}{0.9}
\renewcommand{\textfraction}{0.1}
\renewcommand{\floatpagefraction}{0.8}

\icmlkeywords{Machine Learning, ICML}

\vskip 0.3in
]



\printAffiliationsAndNotice{\icmlEqualContribution} % otherwise use the standard text.

\begin{abstract}
Limited Memory Language Models (LMLMs) externalize factual knowledge to a database to enable deletion-based unlearning without retraining. Existing evaluations measure post-deletion correctness in aggregate and cannot tell whether a deleted fact persists through residual parametric memory, alternative retrieval paths, or near-neighbor retrieval artifacts. We propose a causal auditing framework that holds the model fixed and varies the database state at inference time across three interventions: \texttt{FULL}, \texttt{DEL-ON}, and \texttt{DEL-OFF}. The framework decomposes post-deletion behavior into parametric leakage $L(f)$, retrieval-mediated correctness $R(f)$, and a retrieval artifact rate grounded in the inference-time retrieval trace. We apply it to $12{,}228$ alias-closure deletions across thirteen databases, including four adversarial topologies (\textit{Base}, \textit{Alias}, \textit{Noise}, \textit{Collision}) we construct in three domains, and six prompt formulations. Parametric leakage is near zero in every variant and every prompt style: the model rarely returns the deleted answer in the absence of retrieval. The residual that does survive lives in the retrieval graph: retrieval-mediated correctness and the retrieval artifact rate match within rounding everywhere, so post-deletion correctness is, in our audit, predominantly reconstituted from near-neighbor retrieval. This residual ranges from $0.7\%$ on the released LMLM database to $13.6\%$ on the most adversarial variant, and prompt formulation does not independently control how much of a deleted fact survives. These results suggest that, for this class of LMLM and deletion procedure, the unlearning boundary is drawn primarily by the database administrator rather than by the model.
\end{abstract}

\section{Introduction}
\label{sec:introduction}

Modern language models increasingly rely on hybrid architectures that combine parametric knowledge with external memory. Limited Memory Language Models (LMLMs) are a prominent example of this paradigm, explicitly separating linguistic competence encoded in model parameters from factual knowledge stored in an external database \cite{zhao2025lmlm}. As illustrated in Figure~\ref{fig:lmlm}, an LMLM retains the linguistic competence of a standard language model but routes factual recall through an external database rather than holding it in parameters. This design enables deletion-based unlearning, where removing entries from the database is intended to eliminate access to specific facts without requiring retraining \cite{zhao2025lmlm}. Such capabilities are particularly important for applications involving data governance, privacy, and model editing.

\begin{figure}[H]
    \centering
    \includegraphics[width=1\columnwidth]{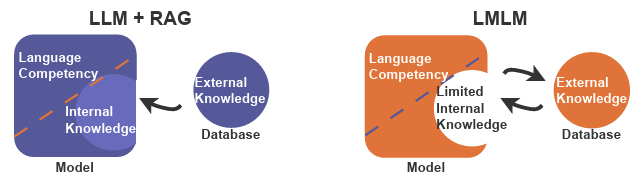}
    \caption{Comparison of a standard retrieval-augmented language model (LLM + RAG) and a LMLM. Both architectures pair a parametric model with an external database, but LMLMs are pre-trained to limit the internal storage of factual knowledge, so factual recall is routed through the external store rather than reconstructed from parameters \cite{zhao2025lmlm}.}
    \label{fig:lmlm}
\end{figure}

However, it remains unclear whether deletion in these systems truly removes knowledge. Existing evaluations of forgetting typically measure whether a model produces the correct answer before and after deletion, but do not distinguish the underlying mechanism of post-deletion correctness \cite{zhao2025lmlm}. A model may still answer correctly due to residual parametric memory, alternative retrieval paths, or semantically related matches in the external database. As a result, current metrics cannot determine whether knowledge has been successfully externalized or whether it persists internally in the model.

In this work, we propose a causal auditing framework for analyzing forgetting in LMLMs. Our approach introduces three controlled interventions, FULL, where the database is intact and retrieval is enabled, DEL-ON, where relevant entries are deleted while retrieval remains enabled, and DEL-OFF, where the same deletion is applied and retrieval is disabled. By comparing model behavior across these conditions, we isolate the contribution of external retrieval and quantify residual internal knowledge. This enables a decomposition of post-deletion behavior into parametric leakage, retrieval-mediated correctness, and retrieval artifacts.

We implement this framework using facts drawn directly from the LMLM database \cite{zhao2025lmlmcode} and other self-made databases\footnote{Code and custom databases are available at \url{https://github.com/raeesiarya/LMLMAudit}.}, and apply a verified deletion procedure that removes all canonical and alias-equivalent representations of a fact. We further log retrieval traces during inference to attribute model outputs to explicit database evidence. This controlled setup provides a principled method for auditing whether knowledge has been successfully externalized in LMLMs.

\section{Related Work}
\label{sec:related_work}
Our work sits at the intersection of three lines of research. First, retrieval-augmented language models pair parametric representations with external memory accessed at inference time \cite{lewis2020rag, guu2020realm}, and the LMLM architecture we audit \cite{zhao2025lmlm} is a recent extension that pre-trains the model to limit internal storage of factual knowledge so that retrieval becomes the primary factual channel. This line of work shows that external retrieval can improve factual accuracy \cite{karpukhin2020dpr}, but leaves open the question of whether knowledge is genuinely externalized or partly retained in parameters \cite{mallen2023trust}.

Second, knowledge-editing methods such as ROME \cite{meng2022rome} and MEMIT \cite{meng2023memit} address the same question from the opposite direction: rather than externalizing knowledge, they locate and modify parametric associations in place \cite{yao2023editing}. LMLMs aim for the cleaner separation that ROME and MEMIT bypass, which makes the LMLM setting a natural place to ask whether the separation actually holds.

Third, work on machine unlearning has probed whether specific training data can be removed from a model after the fact \cite{bourtoule2021unlearning}, and has tended to find that residual traces are difficult to fully eliminate \cite{lizzo2026unlearning, carlini2021extracting}. Our audit takes up this question in the LMLM setting, where the question becomes tractable in a new way: because the database is the intended factual store, deletion can be applied at inference time and the resulting behavior decomposed into a parametric channel and a retrieval channel separately. The framework we develop in Section~\ref{sec:method_setup} is, to our knowledge, the first to make that decomposition per-fact and to attribute surviving correctness to specific retrieval candidates.

\section{Method Setup: Causal Audit Framework}
\label{sec:method_setup}

We evaluate whether Limited Memory Language Models actually forget facts when those facts are removed from the external database. Each fact is represented as a subject--relation--object tuple, such as \textit{Geri Halliwell -- Famous For -- Spice Girls}. At inference time, retrieval provides the model with relevant database entries as additional context before the model generates an answer. Thus, when retrieval is enabled, the model is not reading the entire database directly; instead, a retrieval step selects relevant facts and inserts them into the model's input context. Appendix~\ref{app:prompt_formulations} gives examples of the prompt formulations used in the experimental grid.

Our audit compares three intervention states. In \texttt{FULL}, the target fact remains in the database and retrieval is enabled. This measures normal database-supported accuracy. In \texttt{DEL-ON}, the target fact is deleted but retrieval remains enabled. This tests whether the answer can still be recovered through alternative database entries, aliases, semantically related facts, or retrieval artifacts. In \texttt{DEL-OFF}, the target fact is deleted and retrieval is disabled. This isolates parametric recall, since the model must answer without retrieved evidence.

We use these interventions to decompose post-deletion correctness into three mechanisms. Parametric leakage occurs when the model answers correctly in \texttt{DEL-OFF}, indicating that the deleted fact may still be stored in the model parameters. Retrieval-mediated correctness occurs when the model is correct in \texttt{DEL-ON} but not in \texttt{DEL-OFF}, indicating that retrieval helped recover the answer after deletion. Retrieval artifacts occur when the model produces the correct answer even though the deleted fact is not directly available as retrieved evidence. Formally, for a fact $f=(s,r,o)$ and intervention condition $c$, let $Y(f,c)$ denote the normalized model prediction. We define parametric leakage as
\[
L(f)=\mathbb{I}[Y(f,\textsc{DEL-OFF})=o],
\]
which indicates whether the model can recover the deleted fact without retrieval.

We define retrieval-mediated correctness as
\[
R(f)=\mathbb{I}[Y(f,\textsc{DEL-ON})=o \wedge Y(f,\textsc{DEL-OFF})\neq o],
\]
which captures cases where retrieval enables the correct answer after deletion.

Across a fact set $\mathcal{F}$, the empirical leakage rate is
\[
\hat{L}=\frac{1}{|\mathcal{F}|}\sum_{f\in\mathcal{F}} L(f)
=
\frac{1}{|\mathcal{F}|}\sum_{f\in\mathcal{F}}
\mathbb{I}[Y(f,\textsc{DEL-OFF})=o].
\]
Analogously, we estimate retrieval-mediated correctness by averaging $R(f)$ over $\mathcal{F}$.

\section{Experimental Setup}
\label{sec:experiment_setup}

In addition to the released LMLM database \cite{zhao2025lmlmcode}, we developed databases for countries, politicians and sports. Within each of these three themes, we built four database variants (\textit{Base}, \textit{Alias}, \textit{Noise}, and \textit{Collision}), for a total of twelve custom databases. These four variants are designed to stress-test a distinct mechanism by which a deleted fact could remain accessible after canonical removal. \textit{Base} contains only the canonical $(s, r, o)$ triplet for each fact, leaving retrieval with no alternative path. \textit{Alias} stores the same fact only under aliased subject and relation forms, with no canonical entry, and tests whether alias-closure deletion catches every surface realization. \textit{Noise} augments Base with decoy triplets that route to the same object through paraphrased subjects (e.g., \emph{Government of United States $\rightarrow$ Seat of Government $\rightarrow$ Washington, D.C.}), probing whether retrieval can recover the deleted answer via near-neighbor paraphrases. \textit{Collision} augments Base with near-miss triplets that share the subject but route to a different object (e.g., \emph{United States $\rightarrow$ Largest City $\rightarrow$ New York City}), probing whether retrieval drifts onto a confusable neighbor and returns a plausible but incorrect answer. Table \ref{tab:custom_dataset_variants} outlines an example of these four variants for the politician domain. We evaluate each target fact under six prompt formulations: direct questions, paraphrased questions, contextual questions, cloze prompts, continuations, and few-shot prompts; examples are provided in Appendix~\ref{app:prompt_formulations}.

Crossing the six prompt sets, the three intervention states, and the thirteen databases yields a fully crossed evaluation grid in which every target fact is scored on matched inputs across all conditions. For each cell we record exact match, token-level precision, recall, and F1, together with the cross-state quantities $L(f)$, $R(f)$, and the retrieval artifact rate; retrieval traces are logged at every \texttt{FULL} and \texttt{DEL-ON} call so that post-deletion correctness can be attributed to explicit database evidence rather than implicit model behavior. Figure \ref{fig:setup_chart} outlines this pipeline.

\begin{figure}[H]
    \centering
    \includegraphics[width=1\columnwidth]{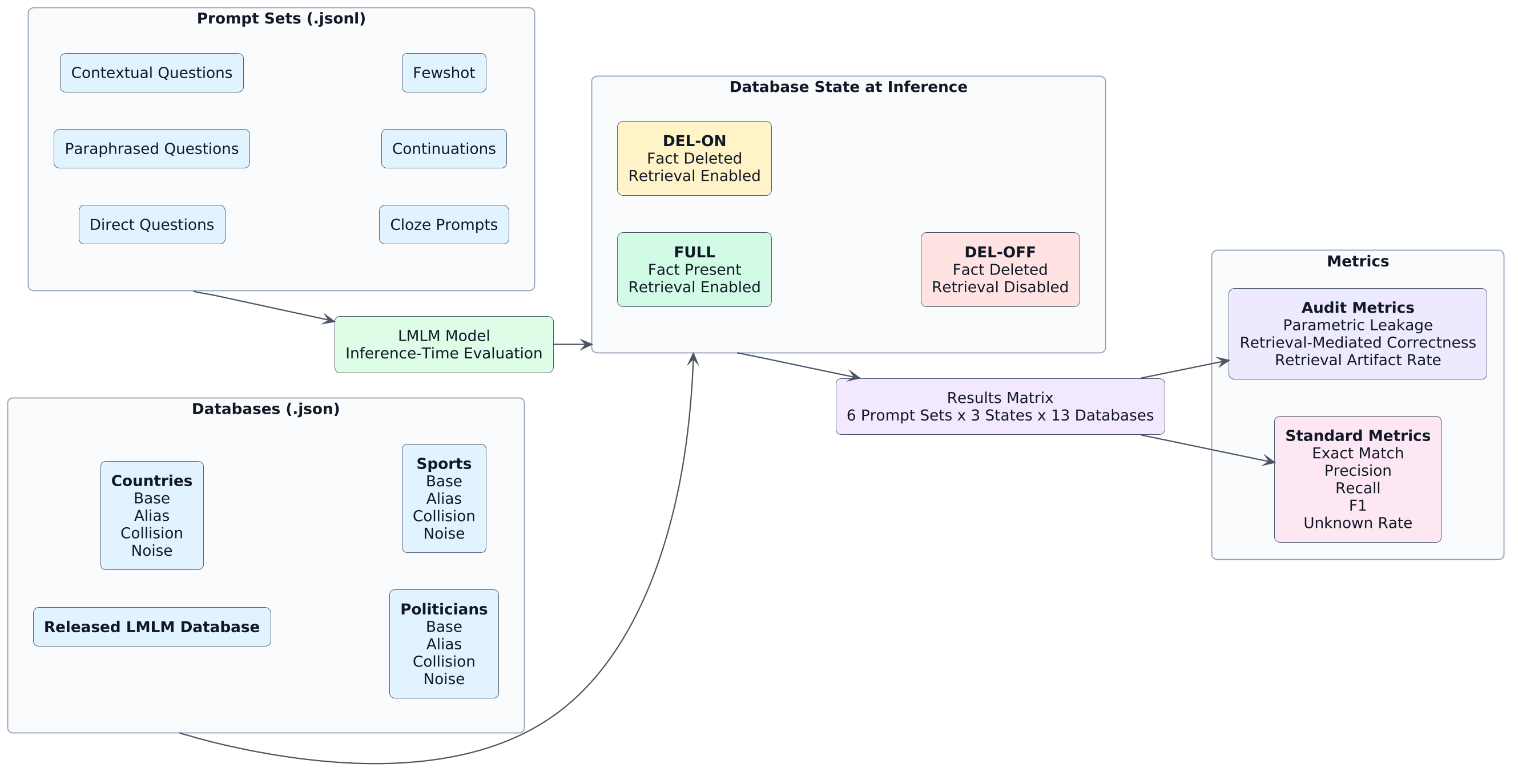}
    \caption{Overview of the evaluation pipeline. We evaluate six prompt sets using the released LMLM database, and specialized databases for countries, politicians and sports. We evaluate the databases under three inference-time database states: \texttt{FULL}, \texttt{DEL-ON}, and \texttt{DEL-OFF}. This yields a $6 \times 3 \times 13$ result matrix for each evaluation metric.}
    \label{fig:setup_chart}
\end{figure}

For each target fact, we construct the deletion set by enumerating the canonical triplet and all alias-equivalent triplets whose subject, relation, and object match the target under our alias mapping. We then remove this set from the database and verify deletion by checking that no retained triplet is gold-equivalent to the target. The same prompt is evaluated under FULL, DEL-ON, and DEL-OFF, and the generated answer is normalized before scoring against the gold object. During FULL and DEL-ON runs, we also save the retrieved candidates returned to the model. A DEL-ON answer is counted as a retrieval artifact when the normalized answer matches the gold object but none of the retained retrieval candidates is gold-equivalent to the deleted fact.

\section{Results}
\label{sec:results}

Our reference point throughout the results is the aggregate evaluation reported in the original LMLM paper~\cite{zhao2025lmlm}, which uses a single FactScore~\cite{min2023factscores} drop when the database is disabled and a single TOFU~\cite{maini2024tofu} forget-quality $p$-value per unlearning step. Both quantities collapse the post-deletion behavior of every fact into a single number, so they cannot indicate which channel a surviving correct answer came through. Our framework refines that aggregate signal into per-fact attributions across the three intervention states, and we read each result below against the corresponding aggregate quantity from the original paper.

We run the LMLM under all three interventions on every cell of the prompt $\times$ database grid. Across the $78$ (prompt file, database) cells, this yields $12{,}228$ paired (\texttt{DEL-ON}, \texttt{DEL-OFF}) evaluations together with an equal number of \texttt{FULL} baselines, for a total of $36{,}684$ model completions. All reported quantities are count-weighted averages over fact-paired groups, so a prompt file with more target facts contributes proportionally to the aggregate. We confine the present section to direct observations and defer cross-figure interpretation to Section~\ref{sec:analysis}.

We begin with the variant-level view. Figure~\ref{fig:del-on-attribution-variant} attributes \texttt{DEL-ON} correctness to its three components separately for the four custom variants and the released LMLM database. In particular, the retrieval artifact bar isolates cases in which \texttt{DEL-ON} returns the gold object even though no gold-equivalent triplet appears among the retrieval-trace candidates retained after deletion.

\begin{figure}[H]
    \centering
    \includegraphics[width=1\columnwidth]{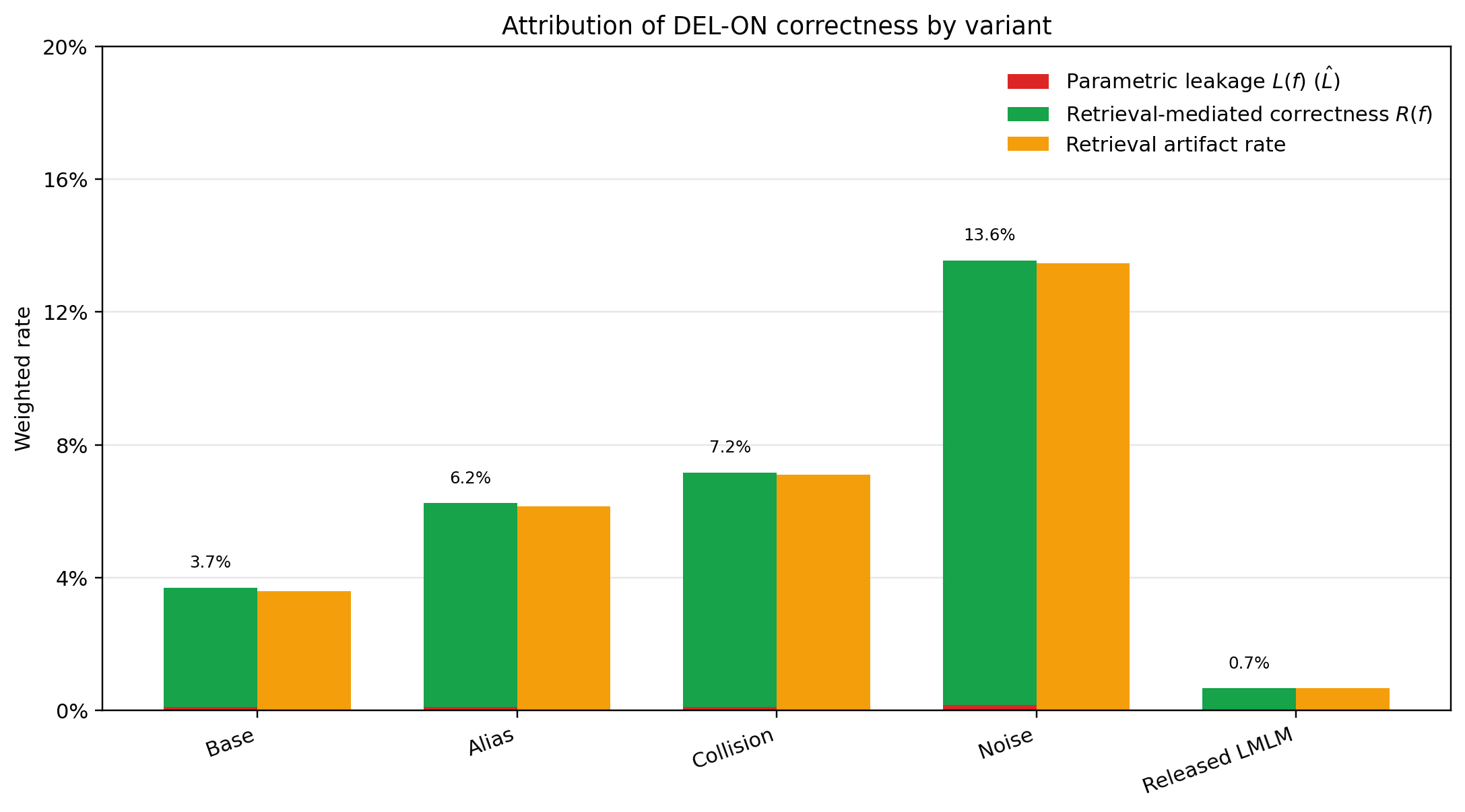}
    \caption{Attribution of \texttt{DEL-ON} correctness by database variant. The stacked left bar in each pair decomposes \texttt{DEL-ON} correctness into parametric leakage $L(f)$ (red, $\hat{L}$) and retrieval-mediated correctness $R(f)$ (green); the orange bar shows the retrieval artifact rate. Parametric leakage is near zero in every variant. The combined $L(f) + R(f)$ stack is $3.7\%$ for \textit{Base}, $6.2\%$ for \textit{Alias}, $7.2\%$ for \textit{Collision}, $13.6\%$ for \textit{Noise}, and $0.7\%$ for the released LMLM database, with the orange artifact bar at approximately the same height as the green bar in each variant.}
    \label{fig:del-on-attribution-variant}
\end{figure}

We next turn to the prompt-style axis. Figure~\ref{fig:del-on-attribution-prompt} reports the same decomposition averaged over all four custom variants and the released LMLM database, so that variant identity is collapsed and only the effect of prompt formulation remains. As before, parametric leakage stays near zero across the panel, although a small red residue is visible under direct and few-shot prompts and is essentially absent under the others. The combined $L(f) + R(f)$ stack varies from $6.1\%$ on cloze prompts to $9.9\%$ on direct questions, while the orange artifact bar matches the green $R(f)$ bar within rounding in every prompt style.

\begin{figure}[H]
    \centering
    \includegraphics[width=1\columnwidth]{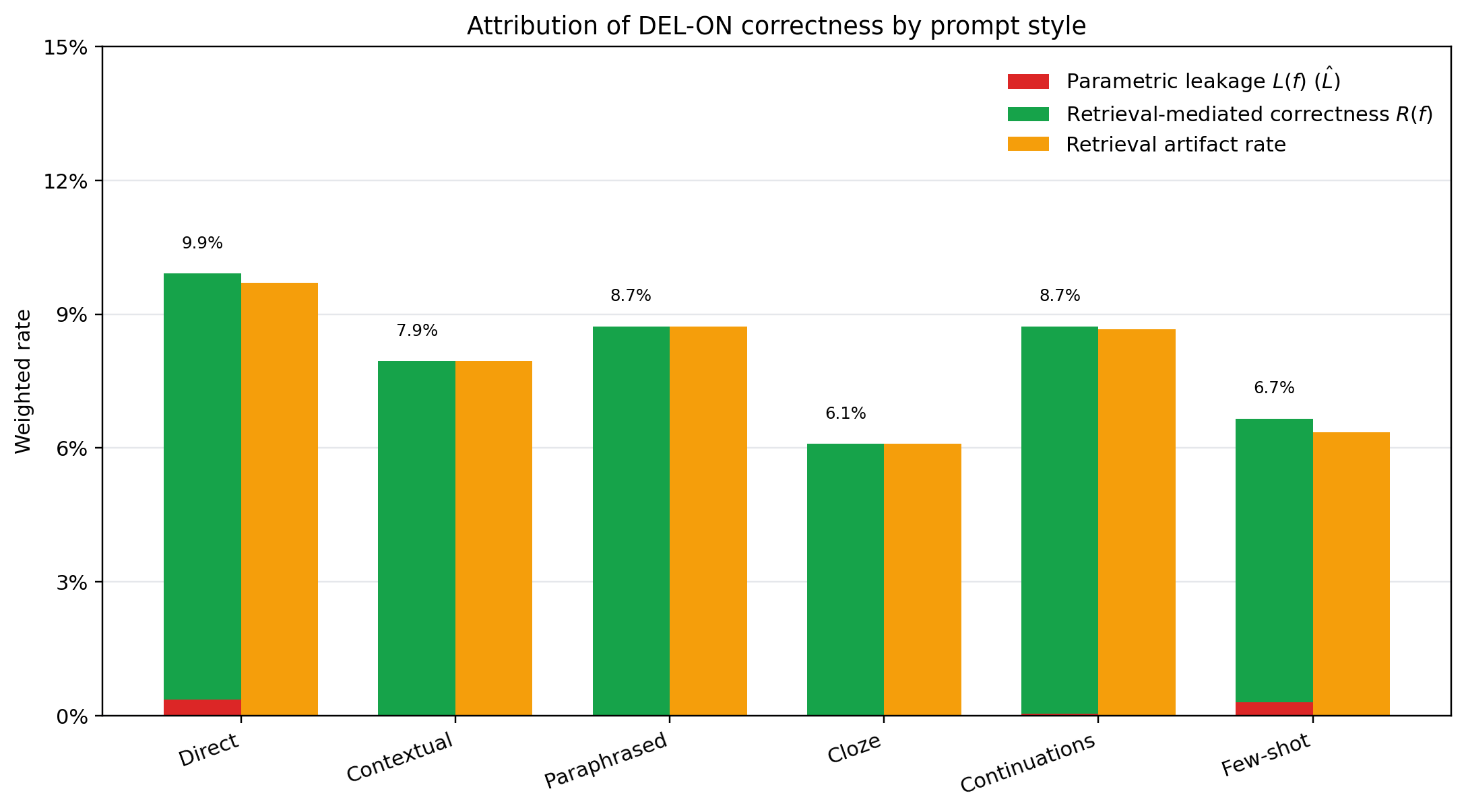}
    \caption{Attribution of \texttt{DEL-ON} correctness by prompt style. The left bar in each pair stacks parametric leakage $L(f)$ (red, $\hat{L}$) and retrieval-mediated correctness $R(f)$ (green); the orange bar shows the retrieval artifact rate. Leakage is near zero across prompt styles, with a small residue under direct and few-shot. The $L(f) + R(f)$ stack is $9.9\%$ (direct), $7.9\%$ (contextual), $8.7\%$ (paraphrased), $6.1\%$ (cloze), $8.7\%$ (continuations), and $6.7\%$ (few-shot); the orange bar matches the green within rounding in every style.}
    \label{fig:del-on-attribution-prompt}
\end{figure}

Having attributed \texttt{DEL-ON} correctness, we now widen the lens to all three intervention states. Figure~\ref{fig:token-f1-prompt-state} reports weighted token F1 by prompt style and state, so that the post-deletion residual can be compared against both the \texttt{FULL} baseline and the retrieval-disabled \texttt{DEL-OFF} condition. Under \texttt{FULL}, F1 spans a wide range, from $\sim$17\% on few-shot prompts to $\sim$57\% on continuations. Under \texttt{DEL-ON}, by contrast, the same prompt styles compress into a narrow band of roughly $7$--$10\%$. As a result, the largest \texttt{FULL}-to-\texttt{DEL-ON} drop falls on continuations (about $48$ points), while the smallest falls on few-shot prompts (about $10$ points). Finally, under \texttt{DEL-OFF}, F1 is visually indistinguishable from zero in every prompt style.

\begin{figure}[htbp]
    \centering
    \includegraphics[width=1\columnwidth]{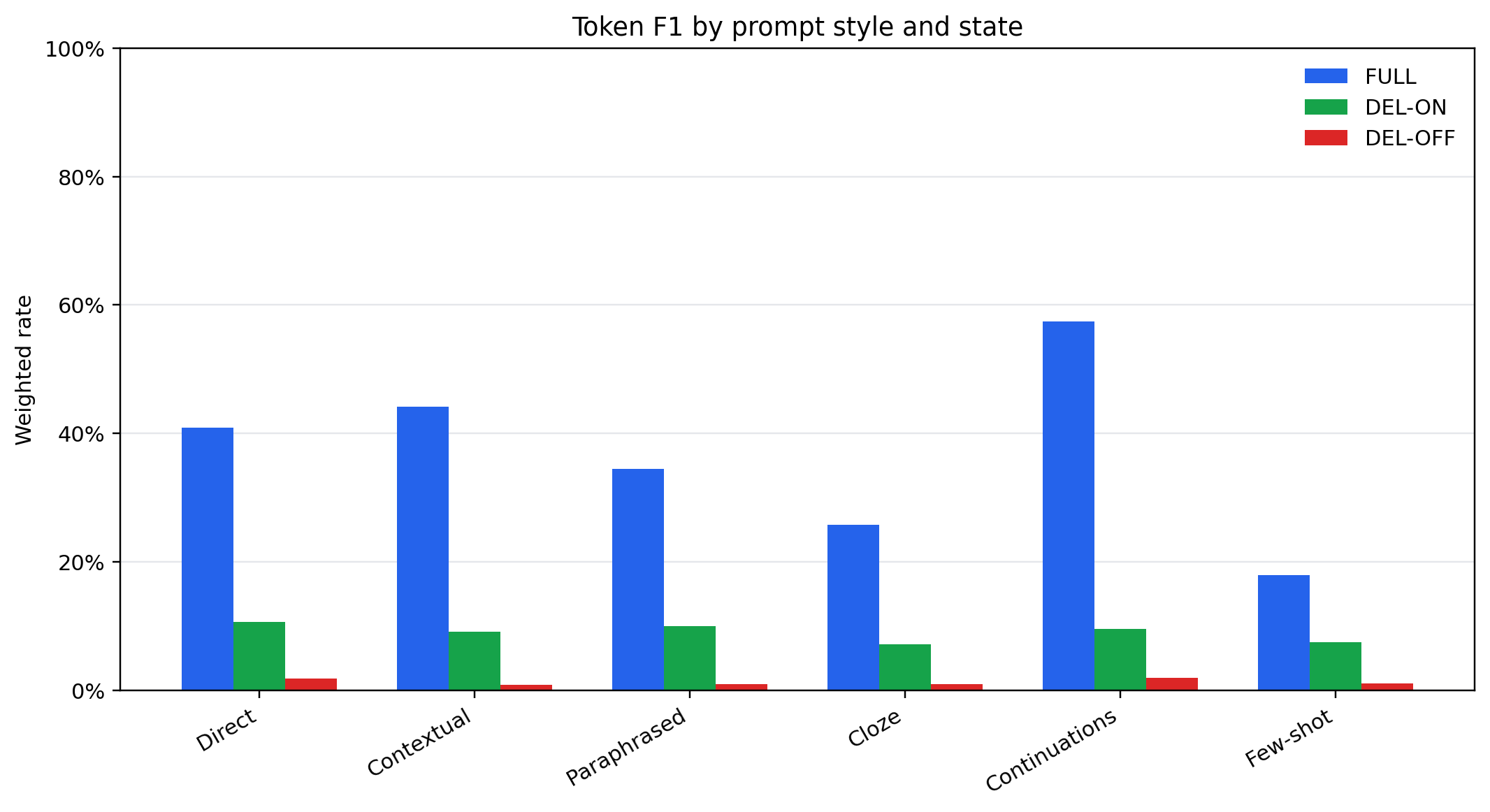}
    \caption{Weighted token F1 by prompt style and intervention state. \texttt{FULL} F1 is $\sim$41\% for direct, $\sim$44\% for contextual, $\sim$34\% for paraphrased, $\sim$26\% for cloze, $\sim$57\% for continuations, and $\sim$17\% for few-shot prompts. \texttt{DEL-ON} F1 sits between $\sim$7\% and $\sim$10\% across all six prompt styles. \texttt{DEL-OFF} F1 is near zero in every prompt style.}
    \label{fig:token-f1-prompt-state}
\end{figure}

Figure~\ref{fig:exact-match-variant-state} mirrors this three-state cut along the database-variant axis and uses exact match in place of token F1, so that the same trajectory can be read off the variants we constructed for the audit. Under \texttt{FULL}, accuracy is similar across \textit{Base}, \textit{Alias}, and \textit{Noise}, all clustered near $33\%$, and noticeably higher on \textit{Collision} and the released LMLM database, both near $46\%$. The corresponding \texttt{FULL}-to-\texttt{DEL-ON} drop is about $30$ points on \textit{Base}, $27$ on \textit{Alias}, $39$ on \textit{Collision}, $19$ on \textit{Noise}, and $45$ on the released LMLM database. Under \texttt{DEL-OFF}, exact match remains at or near zero in every variant.

\begin{figure}[H]
    \centering
    \includegraphics[width=1\columnwidth]{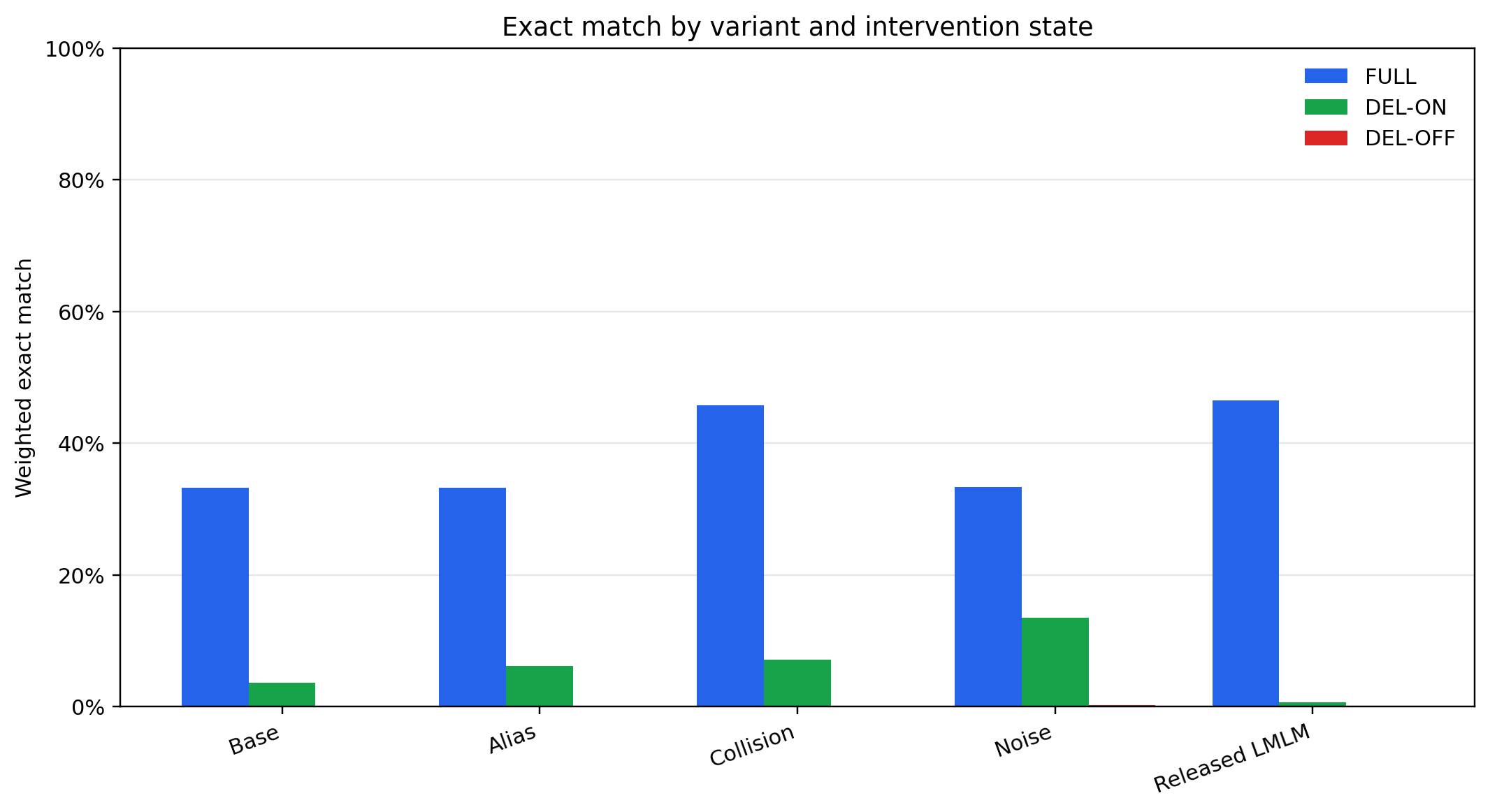}
    \caption{Weighted exact match by database variant and intervention state. \texttt{FULL} exact match is $\sim$33\% on \textit{Base}, \textit{Alias}, and \textit{Noise}, and $\sim$46\% on \textit{Collision} and the released LMLM database. \texttt{DEL-ON} exact match is $\sim$3\% on \textit{Base}, $\sim$6\% on \textit{Alias}, $\sim$7\% on \textit{Collision}, $\sim$14\% on \textit{Noise}, and $\sim$1\% on the released LMLM database. \texttt{DEL-OFF} exact match is near zero in every variant.}
    \label{fig:exact-match-variant-state}
\end{figure}

Tables~\ref{tab:standard_metrics_all_databases} and~\ref{tab:audit_metrics_all_databases} in Appendix~\ref{appex:numeric_results} report the full per-cell numbers underlying Figures~\ref{fig:del-on-attribution-variant}--\ref{fig:exact-match-variant-state}.

\section{Analysis}
\label{sec:analysis}

Sections~\ref{subsec:analysis-leakage} and~\ref{subsec:analysis-residual} establish the headline decomposition of post-deletion correctness into a parametric component and a retrieval-graph component. Sections~\ref{subsec:analysis-topology} and~\ref{subsec:analysis-prompt} then ablate the two axes of the audit grid that the framework holds fixed, the retrieval-graph topology and the prompt formulation, to identify which axis controls the surviving residual. Section~\ref{subsec:error_analysis} closes the analysis with qualitative examples of the three failure modes the decomposition predicts.

\subsection{Parametric externalization at the per-fact level}
\label{subsec:analysis-leakage}

Our headline finding is that parametric leakage $\hat{L}$ is near zero across the audit. The red $L(f)$ component contributes a vanishingly small share of the $L(f) + R(f)$ stack in every variant (Figure~\ref{fig:del-on-attribution-variant}) and in every prompt style (Figure~\ref{fig:del-on-attribution-prompt}), and weighted \texttt{DEL-OFF} performance is visually indistinguishable from zero in every cell of Figures~\ref{fig:token-f1-prompt-state} and~\ref{fig:exact-match-variant-state}. At the level of model parameters, the LMLM checkpoint behaves close to a model that had never seen the fact, in contrast to standard language models, which are known to memorize and expose training-data content even when each example is seen only a handful of times \cite{carlini2021extracting}. The original LMLM paper~\cite{zhao2025lmlm} provides only an aggregate version of this result. Their Table~9 shows that disabling the database globally drops FactScore \cite{min2023factscores} by roughly $19$ points, and their \S5 explicitly labels that finding as ``preliminary evidence'' that LMLM does not memorize. Our audit converts the same intuition into a per-fact statement: across $12{,}228$ alias-closure deletions, the parameters return the deleted answer at a rate of $\hat{L} = 0.11\%$. Our per-fact result is consistent with, and adds resolution to, the preliminary aggregate finding in the original LMLM paper \cite{zhao2025lmlm}.

\subsection{The residual lives in the retrieval graph}
\label{subsec:analysis-residual}

The residual that survives deletion is, in our audit, almost entirely attributable to the retrieval graph rather than to the parameters. The signature observation in Figures~\ref{fig:del-on-attribution-variant} and~\ref{fig:del-on-attribution-prompt} is the near-equality of the green $R(f)$ bar and the orange retrieval artifact bar: the two match within rounding for every variant and every prompt style. The following proposition shows that part of this equality is structural, which in turn clarifies what the figures empirically measure.

\begin{proposition}[Artifact and retrieval-mediated correctness coincide under complete deletion]
\label{prop:artifact-equals-R}
Fix a fact $f = (s, r, o)$ with $L(f) = 0$, and suppose the alias-closure deletion of $f$ is complete, i.e., no triplet alias-equivalent to $(s, r, o)$ survives in the database after deletion. Let $G(f)$ be the indicator that some retained candidate in the inference-time trace is gold-equivalent (alias-equivalent on subject, relation, and object), and let
\[
A(f) = \mathbb{I}\bigl[Y(f, \texttt{DEL-ON}) = o\bigr]\,\bigl(1 - G(f)\bigr)
\]
be the per-fact retrieval artifact indicator. Then $R(f) = A(f)$.
\end{proposition}

\begin{proof}
Since $L(f) = 0$, we have $Y(f, \texttt{DEL-OFF}) \neq o$, so $R(f) = \mathbb{I}[Y(f, \texttt{DEL-ON}) = o]$. By completeness of alias-closure deletion, no retained candidate is gold-equivalent, hence $G(f) = 0$. Therefore $A(f) = \mathbb{I}[Y(f, \texttt{DEL-ON}) = o] = R(f)$.
\end{proof}

Proposition~\ref{prop:artifact-equals-R} reframes what the two attribution figures actually measure. Given the per-fact $\hat{L} \approx 0$ established in Section~\ref{subsec:analysis-leakage}, the empirical near-equality $\hat{R} \approx \hat{A}$ is a confirmation that our alias-closure deletion procedure is complete in practice; the empirical content of the figures is therefore not the equality itself but the topology-dependent magnitude of the residual. Retrieval-mediated correctness almost never coincides with a gold-equivalent retained candidate. Instead, it coincides with near-neighbor candidates that share embedding-space similarity with the deleted entry. This decomposition is not present in the original LMLM paper. Their TOFU \cite{maini2024tofu} evaluation reports a single forget-quality $p$-value per unlearning step and therefore aggregates over both gold-equivalent retrieval and near-neighbor retrieval. The retrieval-trace introspection used here separates the two and shows that, on retrieval graphs containing paraphrastic decoys, the second mechanism dominates. A standard \texttt{FULL} versus \texttt{DEL-ON} comparison would have shown only that some deleted facts are still answered correctly; the audit shows that those facts are not remembered, they are reconstituted at retrieval time.

\subsection{Ablation: Retrieval-graph topology controls the residual}
\label{subsec:analysis-topology}

We ablate the retrieval-graph topology by holding the model, deletion procedure, and prompt distribution fixed and varying only the database variant across \textit{Base}, \textit{Alias}, \textit{Collision}, \textit{Noise}, and the released LMLM database; the dependent variable is the post-deletion residual $L(f) + R(f)$. Reading down Figure~\ref{fig:del-on-attribution-variant}, the residual rises from $3.7\%$ on \textit{Base}, to $6.2\%$ on \textit{Alias}, to $7.2\%$ on \textit{Collision}, and to $13.6\%$ on \textit{Noise}, before falling to $0.7\%$ on the released LMLM database. This ordering tracks the retrieval-graph topologies described in Section~\ref{sec:experiment_setup}: variants in which more surviving entries point to the gold object after alias-closure deletion produce larger residuals, with \textit{Noise} as the worst case by construction and \textit{Base} as the cleanest among the custom databases. The original paper~\cite{zhao2025lmlm} reports unlearning on a single retrieval graph, the annotated TOFU corpus \cite{maini2024tofu}, and on that graph achieves ``ideal forgetting'' with $p > 0.05$. By varying the retrieval graph along controlled topologies, we show that the same architecture and the same alias-closure deletion can produce a \texttt{DEL-ON} residual anywhere from $0.7\%$ to $13.6\%$. The published TOFU result therefore measures the model on a benign retrieval graph \cite{maini2024tofu, zhao2025lmlm}; it does not, by itself, characterize what happens when the graph contains paraphrastic decoys.

\subsection{Ablation: Prompt formulation moves baselines more than residuals}
\label{subsec:analysis-prompt}

We ablate the prompt formulation by holding the model, deletion procedure, and database distribution fixed and varying only the prompt set across the six formulations introduced in Section~\ref{sec:experiment_setup}; the dependent variables are the \texttt{FULL} baseline accuracy and the post-deletion residual under \texttt{DEL-ON}. Under \texttt{FULL}, prompt style produces a wide performance spread (token F1 between $\sim$17\% and $\sim$57\% in Figure~\ref{fig:token-f1-prompt-state}, exact match between $\sim$33\% and $\sim$46\% in Figure~\ref{fig:exact-match-variant-state}), reflecting how easily each prompt style elicits a structured lookup. Under \texttt{DEL-ON}, that spread collapses: token F1 compresses into a $7$--$10\%$ band, and the \texttt{FULL}-to-\texttt{DEL-ON} drop is dominated by the \texttt{FULL} baseline rather than by any property of the prompt itself. Continuations lose the most because they had the most to lose; few-shot prompts lose the least for the same reason. The original paper's TOFU and FactScore evaluations each use a fixed prompt template, so this kind of decomposition is not available there \cite{zhao2025lmlm, maini2024tofu, min2023factscores}. Our six prompt families show that, once the supporting evidence is removed, prompt formulation has limited independent control over how much of a deleted fact survives. What survives is determined almost entirely by the retrieval graph.

\subsection{Error Analysis}
\label{subsec:error_analysis}
The decomposition in Sections~\ref{subsec:analysis-leakage} 
and~\ref{subsec:analysis-residual} predicts three qualitatively 
distinct ways a deleted fact can still be answered correctly: 
parametric leakage, near-neighbor retrieval reconstruction 
on \textit{Noise}, and confusable-neighbor drift on \textit{Collision}. 
Table~\ref{tab:qualitative_examples} in 
Appendix~\ref{app:qualitative} shows one inference-time 
trace for each, with the \textit{Noise} and \textit{Collision} rows 
drawn from the corresponding custom-database files and the 
parametric-leakage row drawn illustratively from a cell that 
empirically contains leakage cases. The traces make the population-level 
rates concrete: in the \textit{Noise} case, alias-closure deletion 
removes the canonical triplet for the gold fact, but a paraphrastic 
decoy with a different subject phrasing survives and routes retrieval 
to the same object, so \texttt{DEL-ON} returns the gold answer with 
no gold-equivalent candidate in the trace; in the \textit{Collision} 
case, deletion is complete and no surviving candidate routes to the 
gold object, but a near-miss triplet sharing the subject pulls retrieval 
onto a confusable neighbor and \texttt{DEL-ON} returns a plausible 
but incorrect answer; and in the parametric-leakage case, retrieval 
returns \texttt{unknown} under \texttt{DEL-OFF} yet the model still 
produces the gold answer from its parameters, an outcome that the 
$\hat{L} = 0.11\%$ headline rate predicts to be rare and that 
appears, in the cells where it occurs, to concentrate on 
high-frequency entities.

\section{Conclusion}
\label{sec:conclusion}
The audit's contribution is not a single number but a decomposition: a retrieval-disabled control isolates parametric recall from retrieval-side correctness, and the retrieval trace separates correctness anchored in gold-equivalent evidence from correctness reconstituted by near-neighbor candidates. Across our grid, the parametric channel is essentially empty in our audit: the LMLM checkpoint almost never returns a deleted answer when retrieval is disabled, and the residual that survives is, in our experiments, primarily a property of the retrieval graph rather than of the model. Within the scope of our audit, LMLM deletion within the model appears to be a clean operation; outside the model, deletion is only as complete as the alias-and-paraphrase closure used to construct it, and our \textit{Noise} variant produces a $13.6\%$ residual at zero parametric leakage precisely by exposing this gap. Prompt formulation does not appear to move this number independently: once the supporting evidence has been removed, what predominantly determines how much of a deleted fact survives is the topology of the retrieval graph rather than the question's surface form.

\section{Limitations and Broader Impact}
\label{sec:limitations_and_broader_impact}

Our audit has several scope limitations. All experiments use a single LMLM checkpoint, the publicly released $382$M-parameter LLaMA2-style model, with retrieval threshold $0.6$, fuzzy-match top-$1$ fallback, and greedy decoding, so findings may not generalize across model sizes, retrieval architectures, similarity thresholds, or decoding strategies \cite{zhao2025lmlmcode}. The custom databases used to stress-test deletion are also small and topology-controlled, on the order of $100$--$240$ facts per (domain, variant) cell, which is what makes paired comparison possible but is several orders of magnitude smaller than the $54.6$M-triplet released LMLM database \cite{zhao2025lmlm}; the variant-level findings therefore describe how deletion behaves when scale is held fixed and topology is varied, rather than how scale itself interacts with topology. We restrict attention to entity-level atomic facts in English, evaluated under an automatic, alias-aware exact-match notion of correctness, so multi-hop or compositional reasoning, languages other than English, sampling-based decoding, and human-judged correctness all remain outside the present scope. Finally, our alias-closure deletion procedure removes only triplets that alias-match the target on subject, relation, and object simultaneously; the \textit{Noise} variant is designed to expose what survives this closure, so its residual is a feature of the audit rather than a methodological shortcoming, but it does mean that low parametric leakage should not be read as a guarantee that any user-facing notion of ``deletion'' fully removes access to sensitive information.

For applications such as data-deletion compliance or factual editing, the practical lesson from the audit is that the unlearning boundary in this class of LMLM is drawn primarily by the database administrator rather than by the model. Sound governance therefore requires a closure procedure that extends beyond entity aliases to retrieval-graph paraphrases of the object.

\section{Future Work}
\label{sec:future_work}

The most direct extension of this work is on the database side. Our results identify retrieval-graph topology as the dominant determinant of post-deletion residuals, which suggests building or preprocessing the database so that alias-closure deletion has fewer surviving routes. One concrete direction is an extended deletion closure that, in addition to alias-equivalent triplets, also removes any triplet whose retained-candidate embedding falls within a similarity radius of the canonical entry, so that paraphrastic decoys are caught at deletion time. A second direction is canonicalization at write time, in which aliases and paraphrastic forms are stored as pointers into a single canonical record rather than as independent triplets. Both approaches are directly testable within our framework: re-running the audit on the modified database and measuring whether $R(f)$ and the retrieval artifact rate fall below their current ranges would tell us how much of the residual is recoverable through database design alone.

Several axes of the audit itself also remain open. We hold the retriever (\textsc{all-MiniLM-L6-v2}) and the similarity threshold ($0.6$) fixed throughout, so sweeping the threshold and replacing the retriever with sparse, dense, or hybrid alternatives would clarify how much of the observed $13.6\%$ \textit{Noise} residual is a property of the embedding model rather than of the LMLM architecture. Scaling the audit to larger LMLM checkpoints and to the full $54.6$M-triplet released database would test whether the per-fact zero parametric leakage we report holds at production scale. Extending the framework beyond entity-level atomic facts to multi-hop and compositional knowledge, and beyond automatic exact-match scoring to semantic or human-judged correctness, would close the gap between the audit and end-user notions of forgetting. Finally, the same intervention set can be applied to retrieval-augmented and edited parametric models more broadly, which would allow a head-to-head comparison of which architectural family achieves the cleanest deletion under matched retrieval pressure.

\section*{Acknowledgements}
\label{sec:acknowledgements}

We thank Akshat Gupta (Ph.D. student, UC Berkeley) for ongoing research feedback and direction; Yilun Hua (Ph.D. student, Cornell University) for further research feedback and direction on the LMLM framework; and Marcel Roed (Ph.D. student, Stanford University) for early feedback on the project proposal. This work used computing resources provided by Berkeley Research Computing through the Compton Spectrometer and Imager (COSI) mission (NASA Small Explorers (SMEX) Program).

\nocite{langley00}

\bibliography{example_paper}
\bibliographystyle{icml2025}

%%%%%%%%%%%%%%%%%%%%%%%%%%%%%%%%%%%%%%%%%%%%%%%%%%%%%%%%%%%%%%%%%%%%%%%%%%%%%%%
%%%%%%%%%%%%%%%%%%%%%%%%%%%%%%%%%%%%%%%%%%%%%%%%%%%%%%%%%%%%%%%%%%%%%%%%%%%%%%%
% APPENDIX
%%%%%%%%%%%%%%%%%%%%%%%%%%%%%%%%%%%%%%%%%%%%%%%%%%%%%%%%%%%%%%%%%%%%%%%%%%%%%%%
%%%%%%%%%%%%%%%%%%%%%%%%%%%%%%%%%%%%%%%%%%%%%%%%%%%%%%%%%%%%%%%%%%%%%%%%%%%%%%%
\newpage
\appendix
\onecolumn
\section{Architecture}
\label{appex:architecture}

%You can have as much text here as you want. The main body must be at most $8$ pages long.
%For the final version, one more page can be added.
%If you want, you can use an appendix like this one.  

%The $\mathtt{\backslash onecolumn}$ command above can be kept in place if you prefer a one-column appendix, or can be removed if you prefer a two-column appendix.  Apart from this possible change, the style (font size, spacing, margins, page numbering, etc.) should be kept the same as the main body.

This appendix details the two design axes that define the prompt $\times$ database grid used throughout the audit: the six prompt formulations applied to every target fact, and the four custom database topologies constructed to stress-test alias-closure deletion. Together these specifications fix the inputs to the $6 \times 3 \times 13$ evaluation grid summarized in Section \ref{sec:experiment_setup}.

\subsection{Prompt formulations}
\label{app:prompt_formulations}

Each target fact is presented to the LMLM under six prompt formulations, designed to vary surface form while holding the underlying fact constant. The formulations span direct question-answering, paraphrased rewordings, contextually framed prompts, cloze-style completions, free-form continuations, and few-shot demonstrations. Their purpose in the audit is to separate the effect of question phrasing on FULL baseline accuracy from any independent effect on the post-deletion residual; in our experiments, prompt formulation moves the FULL baseline considerably more than it moves the DEL-ON residual (Section \ref{subsec:analysis-prompt}). Table~\ref{tab:prompt_examples} illustrates each formulation using a single fact drawn from the released LMLM database.

\begin{longtable}{p{0.28\textwidth} p{0.68\textwidth}}
\caption{Example prompt formulations for the fact ``Geri Halliwell -- Famous For -- Spice Girls'' across the six prompt sets in the released LMLM database.}
\label{tab:prompt_examples} \\

\toprule
\textbf{Prompt Set} & \textbf{Example Prompt} \\
\midrule
\endfirsthead

\multicolumn{2}{l}{\small\emph{Table \thetable{} continued from previous page}} \\
\toprule
\textbf{Prompt Set} & \textbf{Example Prompt} \\
\midrule
\endhead

\midrule
\multicolumn{2}{r}{\small\emph{Continued on next page}} \\
\endfoot

\bottomrule
\endlastfoot

\textbf{Direct} &
What is Geri Halliwell famous for? \\[6pt]

\textbf{Paraphrased} &
Can you tell me what Geri Halliwell is famous for? \\[6pt]

\textbf{Contextual} &
Context: I am compiling a concise factual profile for Geri Halliwell. Answer with a short factual phrase. \newline
Question: What is Geri Halliwell famous for? \\[6pt]

\textbf{Cloze} &
Complete the sentence with the missing fact: Geri Halliwell is famous for \_\_\_\_. \\[6pt]

\textbf{Continuations} &
Tell me about Geri Halliwell. Geri Halliwell is famous for \\[6pt]

\textbf{Fewshot} &
Answer the final question with a short factual phrase. \newline
\newline
Question: Where was Ada Lovelace born? \newline
Answer: London \newline
\newline
Question: In what year was Pride and Prejudice published? \newline
Answer: 1813 \newline
\newline
Question: What is Geri Halliwell famous for? \newline
Answer: \\

\end{longtable}

\subsection{Custom database variants}

Alongside the released LMLM database, we constructed twelve custom databases spanning three domains (countries, politicians, sports) and four topological variants (Base, Alias, Noise, Collision). The variants are intended to isolate distinct mechanisms by which a fact could remain recoverable after canonical removal: \emph{Base} provides a clean baseline with no alternative routes; \emph{Alias} probes whether alias-closure deletion catches every surface realization of the subject and relation; \emph{Noise} probes whether retrieval can reconstitute the deleted answer through paraphrastic decoys that point to the same object; and \emph{Collision} probes whether retrieval drifts onto a near-neighbor sharing the subject but routing to a different object. In our audit, the post-deletion residual tracks this topology ordering, with the largest residual concentrated in the \emph{Noise} variant (Section \ref{subsec:analysis-topology}). Table~\ref{tab:custom_dataset_variants} illustrates the four variants using a single politician-domain example.

\begin{longtable}{p{0.16\textwidth} p{0.34\textwidth} p{0.46\textwidth}}
\caption{Example of the four custom dataset variants used in our experiments, illustrated with a single politician-domain example.}
\label{tab:custom_dataset_variants} \\

\toprule
\textbf{Variant} & \textbf{Description} & \textbf{Example} \\
\midrule
\endfirsthead

\multicolumn{3}{l}{\small\emph{Table \thetable{} continued from previous page}} \\
\toprule
\textbf{Variant} & \textbf{Description} & \textbf{Example} \\
\midrule
\endhead

\midrule
\multicolumn{3}{r}{\small\emph{Continued on next page}} \\
\endfoot

\bottomrule
\endlastfoot

\textbf{Base} &
Clean, straightforward facts with little ambiguity. &
Who is the 47th U.S. president? $\rightarrow$ Donald Trump \\[6pt]

\textbf{Alias} &
Same underlying fact, but expressed with alternate names or phrasings. &
Who is the current POTUS? $\rightarrow$ Donald Trump \newline
Who is the president of the USA? $\rightarrow$ Donald Trump \\[6pt]

\textbf{Noise} &
Many different facts point to the same answer, so that answer becomes very common and tempting for retrieval. &
Who is the 45th U.S. president? $\rightarrow$ Donald Trump \newline
Who is the 47th U.S. president? $\rightarrow$ Donald Trump \newline
Which New York property magnate became U.S. president? $\rightarrow$ Donald Trump \\[6pt]

\textbf{Collision} &
Near-miss facts that are very similar but should produce different answers, so retrieval can confuse neighbors. &
Who is the 46th U.S. president? $\rightarrow$ Joe Biden \newline
Who is the 47th U.S. president? $\rightarrow$ Donald Trump \\
\end{longtable}

\subsection{Qualitative examples of post-deletion failure modes}
\label{app:qualitative}
Table~\ref{tab:qualitative_examples} shows one representative
inference-time trace for each of the three failure modes identified by
the decomposition. For each example we report the gold fact, the
variant the trace was drawn from, the canonical triplet that was
removed by alias-closure deletion, the highest-similarity retrieval
candidate retained at inference time, and the model's \texttt{DEL-ON}
(or \texttt{DEL-OFF}, for the parametric-leakage case) output. The
deleted triplets and the surviving retrieval candidates in the
\textit{Noise} and \textit{Collision} rows are taken verbatim from the
custom databases under
\texttt{data/custom\_databases/countries/}. The parametric-leakage row
is drawn from the \texttt{direct\_questions} $\times$ \textit{Base}
(sports) cell, which contains $2$ leakage cases out of $100$ paired
facts. The specific triplet shown is illustrative of the kind of
high-frequency sports fact this cell contains rather than a verbatim
trace from the audit logs.
\begin{longtable}{p{0.13\textwidth} p{0.10\textwidth} p{0.22\textwidth} p{0.22\textwidth} p{0.20\textwidth}}
\caption{Representative inference-time traces for the three post-deletion failure modes. \textit{Noise} and \textit{Collision} rows reproduce real triplets from the corresponding custom-database files; the parametric-leakage row is drawn from a cell that empirically contains leakage cases (\texttt{direct\_questions} $\times$ \textit{Base} (sports), $2$ of $100$), but the specific triplet shown is illustrative rather than a verbatim trace.}
\label{tab:qualitative_examples} \\
\toprule
\textbf{Failure mode} & \textbf{Variant} & \textbf{Deleted triplet} & \textbf{Surviving retrieval candidate} & \textbf{Model output} \\
\midrule
\endfirsthead
\multicolumn{5}{l}{\small\emph{Table \thetable{} continued from previous page}} \\
\toprule
\textbf{Failure mode} & \textbf{Variant} & \textbf{Deleted triplet} & \textbf{Surviving retrieval candidate} & \textbf{Model output} \\
\midrule
\endhead
\midrule
\multicolumn{5}{r}{\small\emph{Continued on next page}} \\
\endfoot
\bottomrule
\endlastfoot
\textbf{Near-neighbor reconstruction} &
\textit{Noise} (countries) &
\emph{United States $\rightarrow$ Capital $\rightarrow$ Washington, D.C.} &
\emph{Government of United States $\rightarrow$ Seat of Government $\rightarrow$ Washington, D.C.} &
\texttt{DEL-ON}: ``Washington, D.C.'' (gold; no gold-equivalent candidate in trace) \\[6pt]
\textbf{Confusable-neighbor drift} &
\textit{Collision} (countries) &
\emph{United States $\rightarrow$ Capital $\rightarrow$ Washington, D.C.} &
\emph{United States $\rightarrow$ Largest City $\rightarrow$ New York City} &
\texttt{DEL-ON}: ``New York City'' (incorrect; plausible neighbor) \\[6pt]
\textbf{Parametric leakage} &
\textit{Base} (sports) &
\emph{Michael Jordan $\rightarrow$ Sport $\rightarrow$ Basketball} &
None (retrieval returns \texttt{unknown}) &
\texttt{DEL-OFF}: ``Basketball'' (illustrative of the $2$ leakage cases observed in this cell) \\
\end{longtable}

\section{Hyperparameters}
\label{appex:hyperparameters}

This section lists the configuration constants that directly affect the LMLM's outputs in our audit: the loaded checkpoint, the retrieval stack, and the decoding policy. Settings inherited from the upstream LMLM release \citep{zhao2025lmlmcode} are noted as such in the description column. Experimental-grid axes (prompt formulations, database variants, intervention states) and evaluation plumbing are described separately in Appendix~\ref{appex:architecture} and Section~\ref{sec:experiment_setup}. Table~\ref{tab:hyperparameters} lists each performance-relevant hyperparameter, the variable name in source, a short description, and the value used.

{\small
\begin{longtable}{p{0.20\textwidth} p{0.40\textwidth} p{0.34\textwidth}}
\caption{Performance-relevant hyperparameters used throughout the audit. Variable names refer to identifiers in \texttt{src/lmlm-audit/} unless otherwise specified.}
\label{tab:hyperparameters} \\

\toprule
\textbf{Hyperparameter} & \textbf{Description} & \textbf{Value} \\
\midrule
\endfirsthead

\multicolumn{3}{l}{\small\emph{Table \thetable{} continued from previous page}} \\
\toprule
\textbf{Hyperparameter} & \textbf{Description} & \textbf{Value} \\
\midrule
\endhead

\midrule
\multicolumn{3}{r}{\small\emph{Continued on next page}} \\
\endfoot

\bottomrule
\endlastfoot

\textbf{\texttt{model\_name}} &
LMLM checkpoint loaded by the audit's model loader; LLaMA2-style decoder-only model pre-trained from scratch by the upstream authors on annotated Wikipedia. &
\texttt{kilian-group/}\newline\texttt{LMLM-llama2-382M} \\[6pt]

\textbf{logit bias on db tokens} &
Decoding-time logit bias applied to the four lookup tokens so that lookup calls are issued reliably; values inherited from upstream. &
\texttt{<|db\_entity|>}: 4 \newline
\texttt{<|db\_relationship|>}: 2 \newline
\texttt{<|db\_return|>}: 2 \newline
\texttt{<|db\_end|>}: 2 \\[6pt]

\textbf{embedding model} &
Sentence-Transformer used to embed lookup queries and database triplets for top-$k$ FAISS retrieval; inherited from the upstream top-$k$ retriever. &
\texttt{sentence-transformers/}\newline\texttt{all-MiniLM-L6-v2} \\[6pt]

\textbf{\texttt{threshold}} &
Cosine similarity threshold used by the top-$k$ retriever; candidates with score below this are dropped, and retrieval returns \texttt{unknown} if no candidate clears it. &
0.6 \\[6pt]

\textbf{\texttt{fallback\_policy}} &
Behavior when no candidate clears \texttt{threshold}; \texttt{top1\_anyway} re-runs retrieval at threshold $-1.0$ and returns the highest-similarity candidate. &
\texttt{top1\_anyway} \\[6pt]

\textbf{\texttt{max\_new\_tokens}} &
Per-call cap on freshly generated answer tokens, used as the target answer length when computing the generation budget. &
12 \\[6pt]

\textbf{generation budget} &
Total token budget passed to \texttt{model.generate}, sized to leave slack for lookup markup before the retrieved value appears. &
$\max(32,\, |p| + 28)$ \\[6pt]

\textbf{\texttt{repetition\_penalty}} &
HuggingFace \texttt{generate} repetition penalty applied during decoding. &
1.2 \\[6pt]

\textbf{\texttt{do\_sample}} &
Whether to sample tokens during decoding; greedy decoding is used throughout the audit. &
\texttt{False} \\[6pt]

\textbf{\texttt{eos\_token\_id}} &
Stop tokens that terminate generation early so the audit can intercept retrieval at the correct point. &
\texttt{<|db\_return|>} \newline
\texttt{tokenizer.eos\_token\_id} \newline
\texttt{<|end\_of\_text|>} \\

\end{longtable}
}

\section{LMLM Audit Identities}
\label{appex:lmlm_identities}

This section collects the small set of equations that are specific to the
three-state LMLM audit. Each
identity links the per-fact indicators $L(f)$ and $R(f)$ from
Section~\ref{sec:method_setup} and the per-fact retrieval
artifact indicator $A(f)$ from
Proposition~\ref{prop:artifact-equals-R} to the paired contingency that
the McNemar test in Appendix~\ref{appex_subsec:mcnemar} consumes.
Throughout, $\mathcal{F}$ denotes the set of paired facts (those evaluated
under both \texttt{DEL-ON} and \texttt{DEL-OFF}), $n = |\mathcal{F}|$,
and the paired contingency is
\begin{align*}
a &= |\{f \in \mathcal{F} : Y(f,\texttt{DEL-ON}) = o \;\wedge\; Y(f,\texttt{DEL-OFF}) = o\}|,\\
b &= |\{f \in \mathcal{F} : Y(f,\texttt{DEL-ON}) = o \;\wedge\; Y(f,\texttt{DEL-OFF}) \neq o\}|,\\
c &= |\{f \in \mathcal{F} : Y(f,\texttt{DEL-ON}) \neq o \;\wedge\; Y(f,\texttt{DEL-OFF}) = o\}|,\\
d &= |\{f \in \mathcal{F} : Y(f,\texttt{DEL-ON}) \neq o \;\wedge\; Y(f,\texttt{DEL-OFF}) \neq o\}|.
\end{align*}
The intervention-condition variable $c$ from
Section~\ref{sec:method_setup} does not appear elsewhere in
this section, so we reuse the symbol for the McNemar-style off-diagonal
count without ambiguity.

\subsection{Decomposition of \texttt{DEL-ON} correctness}
\label{appex_subsec:identity_decomposition}

\begin{proposition}[\texttt{DEL-ON} correctness decomposition]
\label{prop:del_on_decomposition}
For every fact $f \in \mathcal{F}$,
\begin{equation}
\mathbb{I}[Y(f,\texttt{DEL-ON}) = o] \;=\; L(f) + R(f) \;-\; \mathbb{I}\bigl[Y(f,\texttt{DEL-ON}) \neq o \;\wedge\; Y(f,\texttt{DEL-OFF}) = o\bigr]. \label{eq:del_on_per_fact}
\end{equation}
Aggregating equation~(\ref{eq:del_on_per_fact}) over $\mathcal{F}$ and
writing $p_{\texttt{DEL-ON}}$ for the empirical \texttt{DEL-ON}
exact-match rate gives
\begin{equation}
p_{\texttt{DEL-ON}} \;=\; \hat{L} + \hat{R} - \frac{c}{n}. \label{eq:del_on_aggregate}
\end{equation}
\end{proposition}

\begin{proof}
The four indicators $\mathbb{I}[a\text{-cell}]$, $\mathbb{I}[b\text{-cell}]$,
$\mathbb{I}[c\text{-cell}]$, $\mathbb{I}[d\text{-cell}]$ partition the
joint event space of $(Y(f,\texttt{DEL-ON}), Y(f,\texttt{DEL-OFF}))$, so
for every $f$ exactly one of them is $1$. By construction
\begin{align*}
\mathbb{I}[Y(f,\texttt{DEL-ON}) = o] &= \mathbb{I}[a\text{-cell}] + \mathbb{I}[b\text{-cell}],\\
L(f) = \mathbb{I}[Y(f,\texttt{DEL-OFF}) = o] &= \mathbb{I}[a\text{-cell}] + \mathbb{I}[c\text{-cell}],\\
R(f) &= \mathbb{I}[b\text{-cell}].
\end{align*}
Adding the first two relations and subtracting $\mathbb{I}[c\text{-cell}]$
cancels the duplicated $\mathbb{I}[a\text{-cell}]$ and yields
equation~(\ref{eq:del_on_per_fact}). Summing
equation~(\ref{eq:del_on_per_fact}) over $\mathcal{F}$ and dividing by
$n$ gives equation~(\ref{eq:del_on_aggregate}), since
$\sum_{f \in \mathcal{F}} \mathbb{I}[c\text{-cell}] = c$.
\end{proof}

Proposition~\ref{prop:del_on_decomposition} explains why the
$\hat{L} + \hat{R}$ stack reported in
Figure~\ref{fig:del-on-attribution-variant} reads almost identically to
the \texttt{DEL-ON} exact-match column of
Figure~\ref{fig:exact-match-variant-state}: the gap is exactly $c/n$,
which never exceeds $4/600 \approx 0.7\%$ in any of our variants. The
proposition also pins down what $\hat{L} + \hat{R}$ would equal in a
hypothetical setting where retrieval frequently hurts correctness; in
that setting $c/n$ would no longer be negligible and the stack and the
\texttt{DEL-ON} bar would diverge by exactly that amount.

\subsection{Recovering the paired contingency from aggregated rates}
\label{appex_subsec:identity_contingency}

\begin{proposition}[Contingency recovery]
\label{prop:contingency_recovery}
Suppose every fact $f \in \mathcal{F}$ is evaluated under both
\texttt{DEL-ON} and \texttt{DEL-OFF}. Let $p_{\texttt{DEL-ON}}$ and
$p_{\texttt{DEL-OFF}} = \hat{L}$ denote the empirical \texttt{DEL-ON}
and \texttt{DEL-OFF} exact-match rates and $\hat{R}$ the empirical
retrieval-mediated correctness rate. Then the paired contingency is
\begin{equation}
b = n\hat{R}, \qquad a = n\,p_{\texttt{DEL-ON}} - n\hat{R}, \qquad c = n\,p_{\texttt{DEL-OFF}} - a, \qquad d = n - a - b - c. \label{eq:contingency_recovery}
\end{equation}
\end{proposition}

\begin{proof}
By definition $\hat{R} = (1/n)\sum_{f \in \mathcal{F}} R(f) = b/n$, so
$b = n\hat{R}$. The aggregate \texttt{DEL-ON} exact-match rate is
$p_{\texttt{DEL-ON}} = (a + b)/n$, giving
$a = n\,p_{\texttt{DEL-ON}} - b = n\,p_{\texttt{DEL-ON}} - n\hat{R}$.
The aggregate \texttt{DEL-OFF} exact-match rate is
$p_{\texttt{DEL-OFF}} = \hat{L} = (a + c)/n$, giving
$c = n\,p_{\texttt{DEL-OFF}} - a$. Since $\{a, b, c, d\}$ partition
$\mathcal{F}$, $a + b + c + d = n$ and so $d = n - a - b - c$.
\end{proof}

Proposition~\ref{prop:contingency_recovery} is what allows the
aggregated cross-state and per-state metrics produced by the audit to
be re-expanded into the $(a, b, c, d)$ table that the McNemar statistic
in Appendix~\ref{appex_subsec:mcnemar} consumes, without re-running the
audit on a per-fact basis. The recovery is exact under the same
coverage condition that makes the McNemar test well defined, namely
that every paired fact has been evaluated under both intervention
states.

\section{Additional Results}
\label{appex:additional_results}

\subsection{Paired McNemar test for DEL-ON vs.\ DEL-OFF}
\label{appex_subsec:mcnemar}

Figure~\ref{fig:mcnemar-by-variant} reports the discordant-pair counts that drive the paired McNemar test for \texttt{DEL-ON} against \texttt{DEL-OFF}, broken out by database variant. Both states share the same alias-closure deletion of the target fact; the only difference is that \texttt{DEL-ON} leaves retrieval enabled while \texttt{DEL-OFF} disables it. The count $b$ measures facts that retrieval rescues after the canonical entry is removed, and $c$ measures the reverse transition in which disabling retrieval somehow recovers an otherwise-incorrect answer. The concordant counts $a$ and $d$ are absorbed into the marginal $n$ and do not enter the test.

\begin{figure}[H]
  \centering
  \includegraphics[width=1\linewidth]{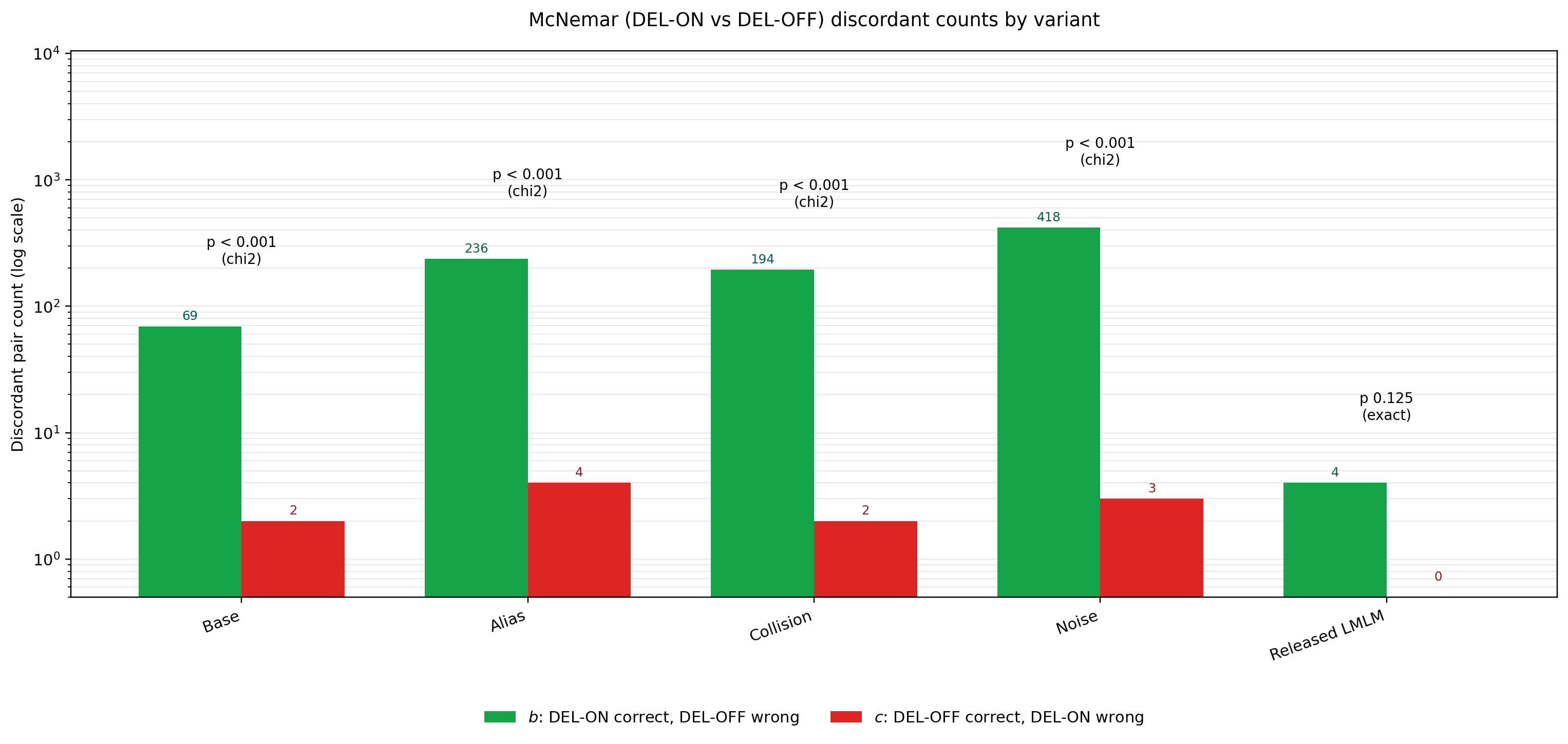}
  \caption{Discordant-pair counts driving the paired McNemar test for \texttt{DEL-ON} against \texttt{DEL-OFF}, broken out by database variant. Bar height is on a log scale so the much smaller $c$ counts remain legible alongside $b$. Annotations above each variant give the $p$-value of the recommended test ($\chi^2$ with continuity correction when $b + c > 25$, exact binomial otherwise).}
  \label{fig:mcnemar-by-variant}
\end{figure}

Across the four custom variants the asymmetry is unambiguous: $b$ ranges from $69$ on \textit{Base} to $418$ on \textit{Noise} while $c$ never exceeds $4$, and all four reject the null at $p < 0.001$ with $\chi^2$ statistics between $61$ and $407$. The Released LMLM column points the same way ($b = 4$, $c = 0$ on $n = 600$), but the discordant counts are too small for the exact binomial test to reject at $\alpha = 0.05$ ($p = 0.125$); we read this as a power limit rather than a contradiction, since the released-database evaluation has both fewer paired facts and a lower base rate of exact-match correctness.

Taken together with the headline rates in the main paper, the McNemar evidence reinforces the claim that retrieval-mediated correctness is a real paired effect rather than a coincidence at the population level. The relation $b \gg c$ holds inside every custom variant, which leaves little room for a reverse hypothesis in which gains from retrieval are merely noise of equal magnitude flipping in both directions.

%%%%%%%%%%

\section{Numeric Results}
\label{appex:numeric_results}

\begin{longtable}{llccccccccc}
\caption{Standard metrics across databases, prompt sets, and inference-time database states. Values are reported as proportions.}
\label{tab:standard_metrics_all_databases} \\

\toprule
& & \multicolumn{3}{c}{\texttt{FULL}} & \multicolumn{3}{c}{\texttt{DEL-ON}} & \multicolumn{3}{c}{\texttt{DEL-OFF}} \\
\cmidrule(lr){3-5} \cmidrule(lr){6-8} \cmidrule(lr){9-11}
Database & Prompt Set & $P$ & $R$ & $F1$ & $P$ & $R$ & $F1$ & $P$ & $R$ & $F1$ \\
\midrule
\endfirsthead

\multicolumn{11}{l}{\small\emph{Table \thetable{} continued from previous page}} \\
\toprule
& & \multicolumn{3}{c}{\texttt{FULL}} & \multicolumn{3}{c}{\texttt{DEL-ON}} & \multicolumn{3}{c}{\texttt{DEL-OFF}} \\
\cmidrule(lr){3-5} \cmidrule(lr){6-8} \cmidrule(lr){9-11}
Database & Prompt Set & $P$ & $R$ & $F1$ & $P$ & $R$ & $F1$ & $P$ & $R$ & $F1$ \\
\midrule
\endhead

\midrule
\multicolumn{11}{r}{\small\emph{Continued on next page}} \\
\endfoot

\bottomrule
\endlastfoot

Released LMLM
& Direct Questions      & 0.474 & 0.477 & 0.475 & 0.022 & 0.022 & 0.021 & 0.015 & 0.024 & 0.016 \\
Released LMLM
& Paraphrased Questions & 0.391 & 0.386 & 0.388 & 0.016 & 0.020 & 0.015 & 0.012 & 0.008 & 0.008 \\
Released LMLM
& Contextual Questions  & 0.594 & 0.597 & 0.594 & 0.032 & 0.039 & 0.031 & 0.008 & 0.031 & 0.011 \\
Released LMLM
& Cloze Prompts         & 0.336 & 0.331 & 0.332 & 0.028 & 0.031 & 0.027 & 0.006 & 0.014 & 0.007 \\
Released LMLM
& Continuations         & 0.815 & 0.819 & 0.816 & 0.026 & 0.034 & 0.026 & 0.028 & 0.016 & 0.013 \\
Released LMLM
& Fewshot               & 0.255 & 0.251 & 0.251 & 0.021 & 0.019 & 0.019 & 0.015 & 0.047 & 0.020 \\
\midrule

\input{tables/per_state_table_rows.tex}

\end{longtable}

\begin{longtable}{llcccc}
\caption{Cross-state audit metrics across databases and prompt sets. Values are reported as proportions.}
\label{tab:audit_metrics_all_databases} \\

\toprule
Database & Prompt Set & Paired Count & $L(f)$ & $R(f)$ & Retrieval Artifact Rate \\
\midrule
\endfirsthead

\multicolumn{6}{l}{\small\emph{Table \thetable{} continued from previous page}} \\
\toprule
Database & Prompt Set & Paired Count & $L(f)$ & $R(f)$ & Retrieval Artifact Rate \\
\midrule
\endhead

\midrule
\multicolumn{6}{r}{\small\emph{Continued on next page}} \\
\endfoot

\bottomrule
\endlastfoot

Released LMLM
& Direct Questions      & 100 & 0.000 & 0.010 & 0.010 \\
Released LMLM
& Paraphrased Questions & 100 & 0.000 & 0.000 & 0.000 \\
Released LMLM
& Contextual Questions  & 100 & 0.000 & 0.010 & 0.010 \\
Released LMLM
& Cloze Prompts         & 100 & 0.000 & 0.010 & 0.010 \\
Released LMLM
& Continuations         & 100 & 0.000 & 0.000 & 0.000 \\
Released LMLM
& Fewshot               & 100 & 0.000 & 0.010 & 0.010 \\
\midrule

\input{tables/cross_state_table_rows.tex}

\end{longtable}

%%%%%%%%%%%%%%%%%%%%%%%%%%%%%%%%%%%%%%%%%%%%%%%%%%%%%%%%%%%%%%%%%%%%%%%%%%%%%%%
%%%%%%%%%%%%%%%%%%%%%%%%%%%%%%%%%%%%%%%%%%%%%%%%%%%%%%%%%%%%%%%%%%%%%%%%%%%%%%%

\end{document}

%% file: tables/per_state_table_rows.tex
Countries Base
& Direct Questions       & 0.580 & 0.580 & 0.580 & 0.098 & 0.098 & 0.098 & 0.012 & 0.030 & 0.017 \\
Countries Base
& Paraphrased Questions  & 0.545 & 0.545 & 0.545 & 0.083 & 0.083 & 0.083 & 0.011 & 0.030 & 0.016 \\
Countries Base
& Contextual Questions   & 0.670 & 0.670 & 0.670 & 0.058 & 0.058 & 0.058 & 0.012 & 0.060 & 0.017 \\
Countries Base
& Cloze Prompts          & 0.311 & 0.315 & 0.311 & 0.022 & 0.027 & 0.023 & 0.004 & 0.040 & 0.007 \\
Countries Base
& Continuations          & 0.840 & 0.840 & 0.840 & 0.028 & 0.028 & 0.028 & 0.042 & 0.180 & 0.067 \\
Countries Base
& Fewshot                & 0.230 & 0.230 & 0.230 & 0.090 & 0.090 & 0.090 & 0.001 & 0.020 & 0.003 \\
\midrule
Countries Alias
& Direct Questions       & 0.757 & 0.757 & 0.757 & 0.217 & 0.217 & 0.217 & 0.006 & 0.043 & 0.011 \\
Countries Alias
& Paraphrased Questions  & 0.566 & 0.566 & 0.566 & 0.152 & 0.149 & 0.150 & 0.008 & 0.056 & 0.014 \\
Countries Alias
& Contextual Questions   & 0.603 & 0.608 & 0.604 & 0.124 & 0.129 & 0.126 & 0.004 & 0.055 & 0.008 \\
Countries Alias
& Cloze Prompts          & 0.511 & 0.516 & 0.512 & 0.113 & 0.115 & 0.112 & 0.004 & 0.028 & 0.007 \\
Countries Alias
& Continuations          & 0.858 & 0.858 & 0.858 & 0.202 & 0.202 & 0.202 & 0.014 & 0.058 & 0.020 \\
Countries Alias
& Fewshot                & 0.253 & 0.253 & 0.253 & 0.092 & 0.092 & 0.092 & 0.010 & 0.045 & 0.015 \\
\midrule
Countries Noise
& Direct Questions       & 0.631 & 0.631 & 0.631 & 0.467 & 0.467 & 0.467 & 0.007 & 0.022 & 0.011 \\
Countries Noise
& Paraphrased Questions  & 0.569 & 0.569 & 0.569 & 0.383 & 0.383 & 0.383 & 0.006 & 0.022 & 0.009 \\
Countries Noise
& Contextual Questions   & 0.670 & 0.670 & 0.670 & 0.389 & 0.389 & 0.389 & 0.009 & 0.050 & 0.014 \\
Countries Noise
& Cloze Prompts          & 0.438 & 0.443 & 0.439 & 0.327 & 0.331 & 0.328 & 0.006 & 0.043 & 0.010 \\
Countries Noise
& Continuations          & 0.878 & 0.878 & 0.878 & 0.354 & 0.354 & 0.354 & 0.025 & 0.128 & 0.041 \\
Countries Noise
& Fewshot                & 0.309 & 0.309 & 0.309 & 0.241 & 0.241 & 0.241 & 0.013 & 0.041 & 0.014 \\
\midrule
Countries Collision
& Direct Questions       & 0.668 & 0.671 & 0.669 & 0.192 & 0.194 & 0.192 & 0.012 & 0.050 & 0.019 \\
Countries Collision
& Paraphrased Questions  & 0.545 & 0.561 & 0.549 & 0.179 & 0.198 & 0.184 & 0.013 & 0.061 & 0.021 \\
Countries Collision
& Contextual Questions   & 0.700 & 0.707 & 0.701 & 0.131 & 0.137 & 0.131 & 0.013 & 0.081 & 0.019 \\
Countries Collision
& Cloze Prompts          & 0.436 & 0.439 & 0.436 & 0.127 & 0.130 & 0.127 & 0.006 & 0.071 & 0.012 \\
Countries Collision
& Continuations          & 0.879 & 0.879 & 0.879 & 0.174 & 0.173 & 0.173 & 0.033 & 0.152 & 0.052 \\
Countries Collision
& Fewshot                & 0.182 & 0.182 & 0.182 & 0.086 & 0.086 & 0.086 & 0.004 & 0.033 & 0.007 \\
\midrule
Politicians Base
& Direct Questions       & 0.504 & 0.504 & 0.504 & 0.033 & 0.033 & 0.033 & 0.000 & 0.000 & 0.000 \\
Politicians Base
& Paraphrased Questions  & 0.496 & 0.493 & 0.493 & 0.121 & 0.121 & 0.119 & 0.005 & 0.025 & 0.008 \\
Politicians Base
& Contextual Questions   & 0.475 & 0.467 & 0.470 & 0.079 & 0.078 & 0.078 & 0.003 & 0.025 & 0.005 \\
Politicians Base
& Cloze Prompts          & 0.283 & 0.279 & 0.281 & 0.042 & 0.042 & 0.040 & 0.010 & 0.060 & 0.016 \\
Politicians Base
& Continuations          & 0.683 & 0.681 & 0.682 & 0.033 & 0.033 & 0.033 & 0.006 & 0.017 & 0.008 \\
Politicians Base
& Fewshot                & 0.304 & 0.304 & 0.302 & 0.104 & 0.101 & 0.100 & 0.007 & 0.037 & 0.011 \\
\midrule
Politicians Alias
& Direct Questions       & 0.387 & 0.385 & 0.386 & 0.053 & 0.054 & 0.054 & 0.003 & 0.015 & 0.004 \\
Politicians Alias
& Paraphrased Questions  & 0.300 & 0.300 & 0.300 & 0.050 & 0.052 & 0.051 & 0.002 & 0.013 & 0.003 \\
Politicians Alias
& Contextual Questions   & 0.565 & 0.567 & 0.565 & 0.075 & 0.078 & 0.075 & 0.002 & 0.022 & 0.004 \\
Politicians Alias
& Cloze Prompts          & 0.256 & 0.254 & 0.255 & 0.054 & 0.048 & 0.050 & 0.001 & 0.013 & 0.002 \\
Politicians Alias
& Continuations          & 0.729 & 0.726 & 0.727 & 0.087 & 0.083 & 0.084 & 0.002 & 0.008 & 0.003 \\
Politicians Alias
& Fewshot                & 0.200 & 0.199 & 0.199 & 0.079 & 0.076 & 0.077 & 0.013 & 0.028 & 0.015 \\
\midrule
Politicians Noise
& Direct Questions       & 0.000 & 0.000 & 0.000 & 0.000 & 0.000 & 0.000 & 0.015 & 0.022 & 0.016 \\
Politicians Noise
& Paraphrased Questions  & 0.000 & 0.003 & 0.001 & 0.000 & 0.003 & 0.001 & 0.008 & 0.036 & 0.013 \\
Politicians Noise
& Contextual Questions   & 0.000 & 0.000 & 0.000 & 0.000 & 0.000 & 0.000 & 0.009 & 0.031 & 0.010 \\
Politicians Noise
& Cloze Prompts          & 0.000 & 0.000 & 0.000 & 0.000 & 0.000 & 0.000 & 0.010 & 0.068 & 0.016 \\
Politicians Noise
& Continuations          & 0.000 & 0.000 & 0.000 & 0.000 & 0.000 & 0.000 & 0.004 & 0.011 & 0.006 \\
Politicians Noise
& Fewshot                & 0.000 & 0.003 & 0.000 & 0.000 & 0.003 & 0.000 & 0.005 & 0.028 & 0.008 \\
\midrule
Politicians Collision
& Direct Questions       & 0.444 & 0.443 & 0.443 & 0.059 & 0.058 & 0.059 & 0.000 & 0.000 & 0.000 \\
Politicians Collision
& Paraphrased Questions  & 0.381 & 0.379 & 0.379 & 0.119 & 0.119 & 0.118 & 0.004 & 0.019 & 0.006 \\
Politicians Collision
& Contextual Questions   & 0.503 & 0.493 & 0.497 & 0.103 & 0.101 & 0.102 & 0.002 & 0.019 & 0.004 \\
Politicians Collision
& Cloze Prompts          & 0.256 & 0.256 & 0.256 & 0.025 & 0.025 & 0.025 & 0.008 & 0.045 & 0.012 \\
Politicians Collision
& Continuations          & 0.600 & 0.598 & 0.599 & 0.075 & 0.075 & 0.075 & 0.005 & 0.019 & 0.008 \\
Politicians Collision
& Fewshot                & 0.278 & 0.276 & 0.275 & 0.075 & 0.072 & 0.071 & 0.005 & 0.028 & 0.008 \\
\midrule
Sports Base
& Direct Questions       & 0.010 & 0.010 & 0.010 & 0.010 & 0.010 & 0.010 & 0.047 & 0.095 & 0.058 \\
Sports Base
& Paraphrased Questions  & 0.010 & 0.010 & 0.010 & 0.010 & 0.010 & 0.010 & 0.005 & 0.020 & 0.008 \\
Sports Base
& Contextual Questions   & 0.010 & 0.010 & 0.010 & 0.010 & 0.010 & 0.010 & 0.004 & 0.048 & 0.006 \\
Sports Base
& Cloze Prompts          & 0.010 & 0.010 & 0.010 & 0.010 & 0.010 & 0.010 & 0.005 & 0.033 & 0.008 \\
Sports Base
& Continuations          & 0.010 & 0.010 & 0.010 & 0.010 & 0.010 & 0.010 & 0.013 & 0.055 & 0.020 \\
Sports Base
& Fewshot                & 0.000 & 0.005 & 0.001 & 0.000 & 0.005 & 0.001 & 0.002 & 0.018 & 0.003 \\
\midrule
Sports Alias
& Direct Questions       & 0.000 & 0.000 & 0.000 & 0.000 & 0.000 & 0.000 & 0.009 & 0.024 & 0.013 \\
Sports Alias
& Paraphrased Questions  & 0.000 & 0.000 & 0.000 & 0.000 & 0.000 & 0.000 & 0.008 & 0.028 & 0.012 \\
Sports Alias
& Contextual Questions   & 0.000 & 0.000 & 0.000 & 0.000 & 0.000 & 0.000 & 0.004 & 0.028 & 0.006 \\
Sports Alias
& Cloze Prompts          & 0.000 & 0.000 & 0.000 & 0.000 & 0.000 & 0.000 & 0.006 & 0.025 & 0.008 \\
Sports Alias
& Continuations          & 0.000 & 0.000 & 0.000 & 0.000 & 0.000 & 0.000 & 0.008 & 0.033 & 0.012 \\
Sports Alias
& Fewshot                & 0.000 & 0.005 & 0.001 & 0.000 & 0.005 & 0.001 & 0.008 & 0.018 & 0.009 \\
\midrule
Sports Noise
& Direct Questions       & 0.475 & 0.466 & 0.468 & 0.072 & 0.061 & 0.063 & 0.041 & 0.092 & 0.053 \\
Sports Noise
& Paraphrased Questions  & 0.405 & 0.406 & 0.406 & 0.054 & 0.050 & 0.050 & 0.006 & 0.028 & 0.009 \\
Sports Noise
& Contextual Questions   & 0.569 & 0.569 & 0.569 & 0.044 & 0.041 & 0.042 & 0.006 & 0.055 & 0.010 \\
Sports Noise
& Cloze Prompts          & 0.284 & 0.274 & 0.276 & 0.053 & 0.043 & 0.045 & 0.004 & 0.027 & 0.007 \\
Sports Noise
& Continuations          & 0.716 & 0.707 & 0.709 & 0.075 & 0.060 & 0.063 & 0.009 & 0.041 & 0.014 \\
Sports Noise
& Fewshot                & 0.196 & 0.195 & 0.193 & 0.058 & 0.058 & 0.055 & 0.006 & 0.041 & 0.010 \\
\midrule
Sports Collision
& Direct Questions       & 0.479 & 0.481 & 0.480 & 0.036 & 0.032 & 0.032 & 0.034 & 0.074 & 0.043 \\
Sports Collision
& Paraphrased Questions  & 0.404 & 0.405 & 0.404 & 0.042 & 0.043 & 0.042 & 0.003 & 0.013 & 0.005 \\
Sports Collision
& Contextual Questions   & 0.544 & 0.540 & 0.541 & 0.045 & 0.040 & 0.039 & 0.002 & 0.032 & 0.005 \\
Sports Collision
& Cloze Prompts          & 0.278 & 0.279 & 0.277 & 0.050 & 0.049 & 0.048 & 0.004 & 0.025 & 0.007 \\
Sports Collision
& Continuations          & 0.696 & 0.696 & 0.696 & 0.054 & 0.046 & 0.049 & 0.009 & 0.038 & 0.013 \\
Sports Collision
& Fewshot                & 0.198 & 0.207 & 0.200 & 0.068 & 0.078 & 0.070 & 0.011 & 0.028 & 0.013 \\

%% file: tables/cross_state_table_rows.tex
Countries Base
& Direct Questions       & 100 & 0.000 & 0.080 & 0.080 \\
Countries Base
& Paraphrased Questions  & 100 & 0.000 & 0.060 & 0.060 \\
Countries Base
& Contextual Questions   & 100 & 0.000 & 0.040 & 0.040 \\
Countries Base
& Cloze Prompts          & 100 & 0.000 & 0.010 & 0.010 \\
Countries Base
& Continuations          & 100 & 0.000 & 0.010 & 0.010 \\
Countries Base
& Fewshot                & 100 & 0.000 & 0.090 & 0.090 \\
\midrule
Countries Alias
& Direct Questions       & 200 & 0.000 & 0.200 & 0.200 \\
Countries Alias
& Paraphrased Questions  & 200 & 0.000 & 0.135 & 0.135 \\
Countries Alias
& Contextual Questions   & 200 & 0.000 & 0.110 & 0.110 \\
Countries Alias
& Cloze Prompts          & 200 & 0.000 & 0.090 & 0.090 \\
Countries Alias
& Continuations          & 200 & 0.005 & 0.185 & 0.185 \\
Countries Alias
& Fewshot                & 200 & 0.000 & 0.090 & 0.090 \\
\midrule
Countries Noise
& Direct Questions       & 180 & 0.000 & 0.456 & 0.456 \\
Countries Noise
& Paraphrased Questions  & 180 & 0.000 & 0.367 & 0.367 \\
Countries Noise
& Contextual Questions   & 180 & 0.000 & 0.378 & 0.378 \\
Countries Noise
& Cloze Prompts          & 180 & 0.000 & 0.322 & 0.322 \\
Countries Noise
& Continuations          & 180 & 0.000 & 0.339 & 0.339 \\
Countries Noise
& Fewshot                & 180 & 0.011 & 0.233 & 0.233 \\
\midrule
Countries Collision
& Direct Questions       & 140 & 0.000 & 0.164 & 0.164 \\
Countries Collision
& Paraphrased Questions  & 140 & 0.000 & 0.143 & 0.143 \\
Countries Collision
& Contextual Questions   & 140 & 0.000 & 0.107 & 0.107 \\
Countries Collision
& Cloze Prompts          & 140 & 0.000 & 0.114 & 0.114 \\
Countries Collision
& Continuations          & 140 & 0.000 & 0.150 & 0.150 \\
Countries Collision
& Fewshot                & 140 & 0.000 & 0.079 & 0.079 \\
\midrule
Politicians Base
& Direct Questions       & 120 & 0.000 & 0.025 & 0.025 \\
Politicians Base
& Paraphrased Questions  & 120 & 0.000 & 0.100 & 0.100 \\
Politicians Base
& Contextual Questions   & 120 & 0.000 & 0.050 & 0.050 \\
Politicians Base
& Cloze Prompts          & 120 & 0.000 & 0.017 & 0.017 \\
Politicians Base
& Continuations          & 120 & 0.000 & 0.033 & 0.033 \\
Politicians Base
& Fewshot                & 120 & 0.000 & 0.067 & 0.067 \\
\midrule
Politicians Alias
& Direct Questions       & 240 & 0.000 & 0.046 & 0.046 \\
Politicians Alias
& Paraphrased Questions  & 240 & 0.000 & 0.042 & 0.042 \\
Politicians Alias
& Contextual Questions   & 240 & 0.000 & 0.067 & 0.067 \\
Politicians Alias
& Cloze Prompts          & 240 & 0.000 & 0.033 & 0.033 \\
Politicians Alias
& Continuations          & 240 & 0.000 & 0.075 & 0.075 \\
Politicians Alias
& Fewshot                & 240 & 0.008 & 0.046 & 0.046 \\
\midrule
Politicians Noise
& Direct Questions       & 180 & 0.006 & 0.000 & 0.000 \\
Politicians Noise
& Paraphrased Questions  & 180 & 0.000 & 0.000 & 0.000 \\
Politicians Noise
& Contextual Questions   & 180 & 0.000 & 0.000 & 0.000 \\
Politicians Noise
& Cloze Prompts          & 180 & 0.000 & 0.000 & 0.000 \\
Politicians Noise
& Continuations          & 180 & 0.000 & 0.000 & 0.000 \\
Politicians Noise
& Fewshot                & 180 & 0.000 & 0.000 & 0.000 \\
\midrule
Politicians Collision
& Direct Questions       & 160 & 0.000 & 0.050 & 0.050 \\
Politicians Collision
& Paraphrased Questions  & 160 & 0.000 & 0.100 & 0.100 \\
Politicians Collision
& Contextual Questions   & 160 & 0.000 & 0.081 & 0.081 \\
Politicians Collision
& Cloze Prompts          & 160 & 0.000 & 0.019 & 0.019 \\
Politicians Collision
& Continuations          & 160 & 0.000 & 0.075 & 0.075 \\
Politicians Collision
& Fewshot                & 160 & 0.000 & 0.044 & 0.044 \\
\midrule
Sports Base
& Direct Questions       & 100 & 0.020 & 0.010 & 0.010 \\
Sports Base
& Paraphrased Questions  & 100 & 0.000 & 0.010 & 0.010 \\
Sports Base
& Contextual Questions   & 100 & 0.000 & 0.010 & 0.010 \\
Sports Base
& Cloze Prompts          & 100 & 0.000 & 0.010 & 0.010 \\
Sports Base
& Continuations          & 100 & 0.000 & 0.010 & 0.010 \\
Sports Base
& Fewshot                & 100 & 0.000 & 0.000 & 0.000 \\
\midrule
Sports Alias
& Direct Questions       & 200 & 0.000 & 0.000 & 0.000 \\
Sports Alias
& Paraphrased Questions  & 200 & 0.000 & 0.000 & 0.000 \\
Sports Alias
& Contextual Questions   & 200 & 0.000 & 0.000 & 0.000 \\
Sports Alias
& Cloze Prompts          & 200 & 0.000 & 0.000 & 0.000 \\
Sports Alias
& Continuations          & 200 & 0.000 & 0.000 & 0.000 \\
Sports Alias
& Fewshot                & 200 & 0.005 & 0.000 & 0.000 \\
\midrule
Sports Noise
& Direct Questions       & 160 & 0.013 & 0.044 & 0.056 \\
Sports Noise
& Paraphrased Questions  & 160 & 0.000 & 0.037 & 0.037 \\
Sports Noise
& Contextual Questions   & 160 & 0.000 & 0.037 & 0.037 \\
Sports Noise
& Cloze Prompts          & 160 & 0.000 & 0.037 & 0.037 \\
Sports Noise
& Continuations          & 160 & 0.000 & 0.050 & 0.050 \\
Sports Noise
& Fewshot                & 160 & 0.000 & 0.050 & 0.050 \\
\midrule
Sports Collision
& Direct Questions       & 158 & 0.013 & 0.013 & 0.019 \\
Sports Collision
& Paraphrased Questions  & 158 & 0.000 & 0.032 & 0.032 \\
Sports Collision
& Contextual Questions   & 158 & 0.000 & 0.025 & 0.025 \\
Sports Collision
& Cloze Prompts          & 158 & 0.000 & 0.032 & 0.032 \\
Sports Collision
& Continuations          & 158 & 0.000 & 0.032 & 0.032 \\
Sports Collision
& Fewshot                & 158 & 0.006 & 0.057 & 0.057 \\